\documentclass{article}
\PassOptionsToPackage{numbers, compress}{natbib}

\usepackage[preprint]{soda_2020}

\usepackage[utf8]{inputenc} 
\usepackage[T1]{fontenc}    
\usepackage[pdfborder={0 0 0}]{hyperref} 
\usepackage{url}            
\usepackage{booktabs}       
\usepackage{amsfonts}       
\usepackage{nicefrac}       
\usepackage{microtype}      
\usepackage{todonotes}
\usepackage{footmisc}
\usepackage{enumitem}

\title{Near-optimal Bayesian Solution For Unknown Discrete Markov Decision Process}
\author{%
	Aristide Tossou, Christos Dimitrakakis, Debabrota Basu\\
	Department of Computer Science and Engineering\\
	Chalmers University of Technology\\
	G{\"o}teborg, Sweden\\
	\texttt{(aristide,chrdimi,basud)@chalmers.se} 
}

\usepackage{times}
\usepackage{algorithm}

\usepackage{graphicx}

\usepackage{bm}

\usepackage{adjustbox}

\usepackage{algorithmicx}
\usepackage{algpseudocode}

\usepackage{amsmath}
\usepackage{amssymb}
\usepackage{amsthm}
\usepackage{amsbsy}
\usepackage{booktabs}
\usepackage{pgfplotstable}
\usepackage{makecell}

\usepackage{mathtools}
\usepackage{subcaption}
\usepackage{float}
\usepackage[separate-uncertainty=true, multi-part-units=single,
detect-all=true]{siunitx}
\usepackage{textcomp}
\usepackage{dsfont}
\usepackage{filecontents}

\newcommand{\BetaDis}{\textsf{Beta}}
\newcommand{\Bernoulli}{\textsf{Bern}}
\newcommand{\Binomial}{\textsf{Binom}}

\DeclareMathOperator{\Prob}{\mathbb{P}}
\DeclareMathOperator{\EX}{\mathbb{E}}

\DeclarePairedDelimiter\abs{\lvert}{\rvert}%
\DeclarePairedDelimiter{\ceil}{\lceil}{\rceil}
\DeclarePairedDelimiter\floor{\lfloor}{\rfloor}

\DeclarePairedDelimiter\set\{\}

\DeclarePairedDelimiter\paren()%
\DeclareMathOperator{\sign}{sgn}
\DeclarePairedDelimiterX{\kldivx}[2]{(}{)}{%
  #1\;\delimsize\|\;#2%
}
\newcommand{\kldiv}{D\kldivx}

\makeatletter
\newcommand{\oset}[3][0ex]{%
	\mathrel{\mathop{#3}\limits^{
			\vbox to#1{\kern-2\ex@
				\hbox{$\scriptstyle#2$}\vss}}}}
\makeatother

\makeatletter
\renewcommand{\Function}[2]{%
	\csname ALG@cmd@\ALG@L @Function\endcsname{#1}{#2}%
	\def\jayden@currentfunction{#1}%
}
\newcommand{\funclabel}[1]{%
	\@bsphack
	\protected@write\@auxout{}{%
		\string\newlabel{#1}{{\jayden@currentfunction}{\thepage}}%
	}%
	\@esphack
}
\makeatother

\makeatletter
\newcommand{\algmargin}{\the\ALG@thistlm}
\makeatother
\algnewcommand{\ParState}[1]{\State%
	\parbox[t]{\dimexpr\linewidth-\algmargin}{\strut\hangindent=\algorithmicindent \hangafter=1 #1\strut}}

\newtheorem{theorem}{Theorem}
\newtheorem{lemma}{Lemma}
\newtheorem{proposition}{Proposition}
\newtheorem{fact}{Fact}
\newtheorem{definition}{Definition}
\theoremstyle{definition}

\theoremstyle{definition}

\newcommand{\UCRL}{{UCRL2}}
\newcommand{\KLUCRL}{{KL-UCRL}}
\newcommand{\PSRL}{{PSRL}}
\newcommand{\TSDE}{{TSDE}}

\newcommand{\UCRLBERNSTEIN}{{UCRL-V}}
\newcommand{\UCRLV}{{UCRL-V}}
\newcommand{\UCRLJEFFREY}{{BUCRL}}

\newcommand{\GAMEOFSKILLEASY}{{GameOfSkill-v1}}
\newcommand{\GAMEOFSKILLHARD}{{GameOfSkill-v2}}

\newcommand{\BUCRL}{{BUCRL}}

\newcommand{\TRIALS}{{\ensuremath{40}}}
\newcommand{\BigO}{{\ensuremath{\mathcal{O}}}}

\newcommand{\TilO}{{\ensuremath{\tilde{\mathcal{O}}}}}
\newcommand{\ConfidenceDelta}{{0.05}}

\newcommand{\gain}{V}
\newcommand{\quant}{Q}

\newcommand{\round}{{round}}
\newcommand{\rounds}{{rounds}}

\newcommand{\episode}{episode}

\newcommand{\regret}{\mathrm{Regret}}

\newcommand{\upperb}[1]{\hat{#1}}
\newcommand{\lowerb}[1]{\check{#1}}
\newcommand{\samples}[1]{\bm{#1}}

\DeclareMathOperator\1{\mathds{1}}

\DeclareMathOperator\rem{\%}

\newcommand{\dbcomment}[1]{}
\newcommand{\aricomment}[1]{}

\usepackage{xr}
\makeatletter
\newcommand*{\addFileDependency}[1]{
  \typeout{(#1)}
  \@addtofilelist{#1}
  \IfFileExists{#1}{}{\typeout{No file #1.}}
}
\makeatother

\makeatletter
\newtheorem*{rep@theorem}{\rep@title}
\newcommand{\newreptheorem}[2]{%
	\newenvironment{rep#1}[1]{%
		\def\rep@title{\textbf{#2} \ref{##1}}%
		\begin{rep@theorem}}%
		{\end{rep@theorem}}}
\makeatother
\newreptheorem{theorem}{Theorem}
\newreptheorem{lemma}{Lemma}
\newreptheorem{proposition}{Proposition}

\makeatletter
\newcommand{\showfontsize}{\f@size{} pt}
\makeatother

\begin{document}
\maketitle

\begin{abstract}
  We tackle the problem of acting in an unknown finite and discrete Markov Decision Process (MDP) for which the expected shortest path from any state to any other state is bounded by a finite number $D$. An MDP consists of $S$ states and $A$ possible actions per state. Upon choosing an action $a_t$ at state $s_t$, one receives a real value reward $r_t$, then one transits to a next state $s_{t+1}$. The reward $r_t$ is generated from a fixed reward distribution depending only on $(s_t, a_t)$ and similarly, the next state $s_{t+1}$ is generated from a fixed transition distribution depending only on $(s_t, a_t)$. The objective is to maximize the accumulated rewards after $T$ interactions. In this paper, we consider the case where the reward distributions, the transitions, $T$ and $D$ are all unknown. We derive the first polynomial time Bayesian algorithm, \BUCRL{} that achieves up to logarithm factors, a regret (i.e the difference between the accumulated rewards of the optimal policy and our algorithm) of the optimal order $\TilO(\sqrt{DSAT})$. Importantly, our result holds with high probability for the worst-case (frequentist) regret and not the weaker notion of Bayesian regret. We perform experiments in a variety of environments that demonstrate the superiority of our algorithm over previous techniques.
  
  Our work also illustrates several results that will be of independent interest. In particular, we derive a sharper upper bound for the KL-divergence of Bernoulli random variables. We also derive sharper upper and lower bounds for Beta and Binomial quantiles. All the bound are very simple and only use elementary functions. 


\end{abstract}

\section{Introduction}
\label{sec:introduction}
Markov Decision Process (MDP) is a framework that is of central importance in computer science. Indeed, MDPs are a generalization of (stochastic) shortest path problems and can thus be used for routing problems \citep{psaraftis2016dynamic}, scheduling and resource allocation problems \citep{gocgun2011markov}. One of its most successful application comes in reinforcement learning where it has been used to achieve human-level performance for a variety of games such as Go \cite{silver2017mastering}, Chess \cite{silver2017masteringchess}. It is also a generalization for online learning problems (such as multi-armed bandit problems) and as such has been used for online advertisement \citep{lu2009showing} and movie recommendations \cite{qin2014contextual}.

\paragraph{Problem Formulation}

In this paper, we focus on the problem of online learning of a near optimal policy for an unknown Markov Decision Process. An MDP consists of $S$ states and $A$ possible actions per state. Upon choosing an action $a_t$ at state $s_t$, one receives a real value reward $r_t$, then one transits to a next state $s_{t+1}$. The reward $r_t$ is generated from a fixed reward distribution depending only on $(s_t, a_t)$ and similarly, the next state $s_{t+1}$ is generated from a fixed transition distribution $p(. | s_t, a_t)$ depending only on $(s_t, a_t)$. The objective is to maximize the accumulated (and undiscounted) rewards after $T$ interactions. An MDP is characterized by a quantity (called $D$) known as the diameter. It indicates an upper bound on the expected shortest path from any state to any other  state. When this diameter (formally defined by Definition \ref{def:diameter}) is finite, the MDP is called \emph{communicating}. 

\begin{definition}[Diameter of an MDP]
\label{def:diameter}
The diameter $D$ of an MDP $M$ is defined as the minimum number of rounds needed to go from one state $s$ and reach any other state $s'$ while acting using some deterministic policy. Formally,
\[D(M) = \max_{s\ne s', s,s' \in \mathcal{S}} \min_{\pi: \mathcal{S} \rightarrow \mathcal{A}} T(s' | s, \pi) \] where $T(s' | s, \pi)$ is the expected number of rounds it takes to reach state $s'$ from $s$ using policy $\pi$.
\end{definition}

 In this paper, we consider the case where the reward distributions $r$, the transitions $p$, $T$ and $D$ are all unknown. %
Given that the rewards are undiscounted, a good measure of performance is the gain, i.e. the infinite
horizon average rewards. The gain of a policy $\pi$ starting from state s is defined by:
\[
\gain(s | \pi) \triangleq \limsup_{T \to \infty}\frac{1}{T} \EX\left[\sum_{t=1}^{T} r(s_t, \pi(s_t)) \mid s_1 = s\right].
\]
\citet{puterman2014markov} shows that there is a policy $\pi^*$ whose gain, $\gain^*$ is greater than that of any other policy. In addition, this gain is the same for all states in a communicating MDP.
We can then characterize the performance of the agent by its regret defined as:
\[
\regret(T) \triangleq \sum_{t=1}^{T} \left(\gain^* - r(s_t, a_t)\right).
\]

Thus our goal is equivalent to obtaining a regret as low as possible.

\paragraph{Related Work}
It has been shown that any algorithm must incur a regret of $\Omega(DSTA)$ in the worst case. \cite{jaksch2010near}. Since the establishment of this lower bound on the regret, there has been numerous algorithms for the problem. They can be classified in two ways: Frequentist and Bayesian. The frequentist algorithms usually construct explicit confidence interval while the Bayesian algorithms start with a prior distribution and uses the posterior derived from Bayes Theorem. Following a long line of algorithms KL-UCRL~\citep{filippi2010optimism}, REGAL.C~\citep{bartlett2009regal}, UCBVI~\citep{azar2017minimax}, SCAL~\citep{fruit2018efficient} the authors of \citep{full_ucrlv_paper} derived a frequentist algorithm that achieved the lower bound up to logarithmic factors.

In contrast, the situation is different for Bayesian algorithms. One of the first to prove theoretical guarantees for posterior sampling is \citet{osband2013more}, for their \PSRL{} algorithm. However, they only consider reinforcement learning problems with a finite and known episode length\footnote{Informally, it is known that the MDP resets to a starting state after a fixed number of steps.} and prove an upper bound of $\BigO(HS\sqrt{TA})$ on the expected Bayesian regret where $H$ is the length of the episode.
\citet{ouyang2017learning} generalises \citet{osband2013more} results to weakly communicating MDP and proves a $\BigO(H_SS\sqrt{TA})$ on the expected Bayesian regret where $H_S$ is a bound on the span of the MDP. Other Bayesian algorithms have also been derived in the litterature however, none of them is able to attain the lower bound for the general communicating MDP considered in this paper. Also many of the previous Bayesian algorithms only provide guarantees about the Bayesian regret (i.e, the regret under the assumption that the true MDP is being sampled from the prior). It was thus an open-ended question whether or not one can design Bayesian algorithms with optimal worst-case regret guarantees\citep{pmlr-v70-osband17a, osband2016posterior}. In this work, we provide guarantees for the worst-case (frequentist) regret. We solve the challenge by designing the first Bayesian algorithm with provable upper bound on the regret that matches the lower bound up to logarithmic factors. Our algorithm start with a prior on MDP and computes the posterior similarly to previous works. However, instead of sampling from the posterior, we compute a quantile from the posterior. We then uses all the MDPs possible under the quantile as a set of statistically plausible MDPs and then follow the same steps as the state-of-the art \UCRLV{} \citep{full_ucrlv_paper}. The idea of using quantiles have already been explored in the algorithm named \emph{Bayes-UCB} \citep{Kaufmann12onbayesian} for multi-armed bandit (a special case of MDP where there is only one single state). Our work can also be considered as a generalization to \emph{Bayes-UCB}.

\paragraph{Our Contributions.} Hereby, we summarise the contributions of this paper that we elaborate in the upcoming sections.
\begin{itemize}

\item We provide a conceptually simple Bayesian algorithm \UCRLJEFFREY{} for reinforcement learning that achieves near-optimal worst case regret. Rather than actually sampling from the posterior distribution, we simply construct upper confidence bounds through Bayesian quantiles.

\item Based on our analysis, we explain why Bayesian approaches are often superior in performance than ones based on concentration inequalities.

\item We perform experiments in a variety of environments that validates the theoretical bounds as well as proves \UCRLJEFFREY{} to be better than the state-of-the-art algorithms. (Section~\ref{sec:experiments})
\end{itemize}
We conclude by summarising the techniques involved in this paper and discussing the possible future works they can lead to (Section~\ref{sec:conclusion}).

\section{Algorithms Description and Analysis}
\label{sec:algorithms}

In this section, we describe our Bayesian algorithm \UCRLJEFFREY{}. We combine Bayesian priors and posterior together with optimism in the face of uncertainty to achieve a high probability upper bound of $\TilO(\sqrt{DSAT})$\footnote{$\TilO$ is used to hide log factors.} on the worst-case regret in any finite communicating MDP. Our algorithm can be summarized as follow:

\begin{enumerate}
	\item Consider a prior distribution over MDPs and update the prior after each observation
	\item Construct a set of statistically plausible MDPs using the set of all MDPs inside a Quantile of the posterior distribution.
	\item Compute a policy (called \emph{optimistic}) whose gain is the maximum among all MDPs in the plausible set. We used a \emph{modified extended value iteration} algorithm derived in \citep{full_ucrlv_paper}.
	\item Play the computed \emph{optimistic} policy for an artificial \episode{} that lasts until the average number of times state-action pairs has been doubled reaches 1. This is known as the \emph{extended doubling trick} \citep{full_ucrlv_paper}.
\end{enumerate} 

They are multiple variants of quantiles definition for MDP (since an unknown MDP can be viewed as a multi-variate random variable). In this paper, we adopt a specific definition of quantiles for multi-variate random variable called \emph{marginal quantiles}. More precisely,
\begin{definition}[Marginal Quantile \cite{babu1989joint}]
  Let $\bm{X} = (\bm{X_1} \ldots \bm{X_m})$ be a multivariate random vector with joint d.f.( distribution function) $F$, the i-th marginal d.f. $F_i$. We denote the ith marginal quantile function by:
  \[\quant_i(F,q) = \inf\{x:F_i(x) \geq q\}, 0 \leq q \leq 1.\]
\end{definition}
Unless otherwise specified, we will refer to marginal quantile as simply \emph{quantile}. For univariate distributions, the subscript $i$ can be omitted, as the quantile and the marginal quantile coincide.

%
%
%

\begin{algorithm}[ht!]
	\caption{\UCRLJEFFREY{}}
	\label{algo:bayes_ucrl}
	\begin{algorithmic}
		\State \textbf{Input: } Let $\mu_1$ the prior distribution over MDPs. $1-\delta$ are confidence level.
		\State \textbf{Initialization: } Set $t \gets 1$ and observe initial state $s_1$
		\State Set $N_k, N_k(s,a), N_{t_k}(s,a)$ to zero for all $k \geq 0$ and $(s,a)$.
		
		\Statex
		
		\For{episodes $k=1, 2, \ldots$}
		
		\ParState{$t_k \gets t$} \ParState{$N_{t_{k+1}}(s, a) \gets N_{t_{k}}(s, a)\; \forall s,a$}
		\Statex
		
		\State \textbf{Compute optimistic policy $\tilde{\pi}_k$:}
				\State \emph{/*Update the bounds on statistically plausible MDPs*/}
				\State $\lowerb{r}(s,a) \gets Q_{s,a, \samples{r}}(\mu_t, \delta^k_r)$  (lower quantile)
				
				\State $\upperb{r}(s,a) \gets Q_{s,a, \samples{r}}(\mu_t, 1-\delta^k_r)$ (upper quantile)
				\ParState{where $Q_{s,a, \samples{r}}$ is the $i$-th marginal quantile function with $i$ the component corresponding to $(s,a)$ for the rewards.}
				
				\State For any $\mathcal{S}_c \subseteq \mathcal{S}$ use:
				\State $\lowerb{p}(\mathcal{S}_c | s,a) \gets Q_{s,a,\mathcal{S}_c, \samples{p}}(\mu_t, \delta^k_p)$  (lower quantile)
				
				\State $\upperb{p}(\mathcal{S}_c |s,a) \gets Q_{s,a,\mathcal{S}_c, \samples{p}}(\mu_t, 1-\delta^k_p)$  (upper quantile)
				
				\ParState{where $Q_{s,a,\mathcal{S}_c, \samples{p}}$ is the $i$-th marginal quantile function with $i$ the component corresponding to $(s,a)$ and the subset $\mathcal{S}_c$ for the transitions.}
				
				\State \emph{/*Find $\tilde{\pi}_k$ with value $\frac{1}{\sqrt{t_k}}$-close to the optimal*/}
				\State $\tilde{\pi}_k \gets \Call{ExtendedValueIteration}{\lowerb{r}, \upperb{r}, \lowerb{p}, \upperb{p}, \frac{1}{\sqrt{t_k}}}$ (Algorithm 2 in \citet{full_ucrlv_paper}.)
				
				\Statex
				\State \textbf{Execute Policy $\tilde{\pi}_k$:}
				
				\While{ $\sum_{s,a} \frac{N_k(s,a)}{\max\{1, N_{t_k}(s,a)\}} < 1$ }
				\Statex

				\State Play action $a_t$ and observe $r_t, s_{t+1}$. Let $r_t \gets \Bernoulli(r_t)$
				
				\State Increase $N_k$ and $N_k(s_t, a_t)$, $N_{t_{k+1}}(s_t, a_t)$ by 1.
				
				\State Update the posterior $\mu_{t+1}$ using Bayes rule.
				
				\State $t \gets t + 1$
				\EndWhile
		\EndFor
	\end{algorithmic}
\end{algorithm}

Our analysis is based on the choice of a specific prior distribution for MDP with bounded rewards.
\paragraph{Prior Distribution}
We consider two different prior distributions. One for computing lower bound on rewards/transitions, that is when computing $\delta$-marginal quantile. One for computing upper bound on rewards/transitions, that is when computing $1-\delta$-marginal quantile.

For the lower bound, we used independent distribution for the rewards and transitions. We also used independent distribution for the rewards of each state-action $(s,a)$. And independent distribution for the transition from any state-action $(s,a)$ to any next subset of states $\mathcal{S}_c$. The prior distribution for any of those components is a beta distribution of parameter $(0,1)$: $\BetaDis(0, 1)$\footnote{Technically, beta distributions are only defined for parameter strictly greater than 0. In this paper, when the parameter $\epsilon$ is 0, we compute the posterior and the quantiles by considering the limit when $\epsilon$ tends to 0. \label{fn:beta_valid}}.

The situation is similar with the upper bound. However, here the prior distribution for any component is a beta distribution of parameter $(1,0)$: $\BetaDis(1, 0)$\footref{fn:beta_valid}.

\paragraph{Posterior Distribution}
Let's start by assuming that the rewards come from the Bernoulli distribution. For the upper bounds, using Bayes rule, the posterior at \round{} $t_k$ are:

For the rewards of any $(s,a)$:  \[\BetaDis(\alpha + \sum_{t\leq t_k: s_t = (s,a)}  r_t,  \beta + N_{t_k}(s,a) - \sum_{t\leq t_k: s_t = (s,a)}  r_t)\]

For the transitions from any $(s,a)$ to any subset of next state $\mathcal{S}_c$ are:

\[\BetaDis(\alpha + \sum_{t\leq t_k: s_t = (s,a)}  p_t,  \beta + N_{t_k}(s,a) - \sum_{t\leq t_k: s_t = (s,a)}  p_t)\]
where $p_t = 1$ if $s_{t+1} \in \mathcal{S}_c$; $p_t = 0$ otherwise.

$\alpha = 1, \beta = 0$ for the upper posteriors and $\alpha = 0, \beta = 1$ for the lower posteriors.

\paragraph{Dealing with non-Bernoulli rewards}
We deal with non-Bernoulli rewards by performing a Bernoulli trials on the observed rewards. In other words, upon observing $r_t$ we used $\Bernoulli(r_t)$ where $\Bernoulli(r_t)$ is a sample from the Bernoulli distribution of parameter $r_t$. This technique is already used in \cite{agrawal2012analysis} and ensures that our prior remain valid.

\paragraph{Quantiles}
When $N_{t_k}(s,a) = 0$ the lower and upper quantiles are respectively $0$ and $1$. When the first parameter of the posterior is $0$, the lower quantile is $0$. When the second parameter of the posterior is $0$, the upper quantile is $1$. In all other cases, the $\delta$ quantile corresponds to the inverse cumulative distribution function of the posterior at the point $\delta$. To achieve a high probability bound of $1-\delta$ on our regret, we used the following parameters respectively for the rewards and transitions $\delta^k_r = \frac{\delta}{4 SA \ln\paren*{2t}}$, $\delta^k_p = \frac{\delta}{8S^2A \ln\paren*{2t}}$, where $1-\delta$ is the desired confidence level of the set of plausible MDPs.

\begin{theorem}[Upper Bound on the Regret of \UCRLJEFFREY{}]\label{thm:bayes_ucrl}
With probability at least $1-\delta$ for any $\delta \in ]0,1[$, any $T \geq 1$, the regret of \UCRLJEFFREY{} is bounded by:
\begin{align*}
\mathcal{R}(T) &\leq
20\cdot \sqrt{\min\{S,\log_2^2 2D\} D T SA\log T \ln \paren*{\frac{B}{\delta}}} +  9DSA\ln \paren*{\frac{B}{\delta}}
\end{align*}
for $B = 9S\sqrt{TDSA}\ln\paren*{TSA}$. \dbcomment{Aristide: Check the constants once?}
\end{theorem}

\begin{proof}[Proof]
Our proof is based on the generic proof provided in \citet{full_ucrlv_paper}. To apply that generic proof, we need to show that with high probability the true rewards/transitions of any state-action is contained in the lower and upper quantiles of the Bayesian Posterior. In other words we need to show that the Bayesian quantiles provide exact coverage probabilities. For that we notice that our prior lead to the same confidence interval as the Clopper-Pearson interval (See Lemma \ref{lemma:coverage_beta}). Furthermore, we need to provide upper and lower bound for the maximum deviation of the Bayesian posterior quantiles from the empirical values. This is a direct consequence of Proposition \ref{lemma:beta_quantile_bound} and \ref{lemma:beta_quantile_bound_lower}.
\end{proof}

The following results were all useful in establishing our main result in Theorem \ref{thm:bayes_ucrl}. Our main contribution in Proposition \ref{theo:kl_bounds} is the upper bound (the first term of the upper bound) for the KL-divergence of two bernoulli random variables. The last term of the upper bound is a direct derivation from the upper bounds in \citep{dragomir2000some}. Our result in Proposition \ref{theo:kl_bounds} shows a factor of $2$ improvement in the leading term of the upper bound. The KL divergence of Bernoulli random is useful for many online learning problems and we used it here to bound the quantile of the Binomial distributions in term of simple functions.

\begin{repproposition}{theo:kl_bounds}[Bernoulli KL-Divergence]
\input{proposition:kl_bounds}
\end{repproposition}
\begin{proof}[Proof Sketch]
The main idea to prove the upper bound is by studying the sign of the function in $x$ obtained by taking the difference of the KL-divergence and the upper bound. We used Sturm' theorem to basically show that this function starts as a decreasing function then after a point becomes increasing for the remaining of its domain. This together with the observation that at the end of its domain the function is non-positive concludes our proof. Full detailed are available in the appendix.
\end{proof}

Proposition \ref{lemma_binomial_quantile_bounds} provides tight lower and upper bound for the quantile of the binomial distribution in the same simple form as Bernstein inequalities. Binomial distributions and their quantiles are useful for a lot of applications and we use it here to derive the bounds for the quantile from a Beta distribution in Proposition \ref{lemma:beta_quantile_bound} and \ref{lemma:beta_quantile_bound_lower}.
\begin{repproposition}{lemma_binomial_quantile_bounds}[Lower and Upper bound on the Binomial Quantile]
\input{proposition:binom_bounds}
\end{repproposition}
\begin{proof}[Proof Sketch]
We used the tights bounds for the cdf of Binomial in \cite{zubkov2013complete}. We inverted those bounds and then use the upper and lower bound for KL divergence in Proposition \ref{theo:kl_bounds} to conclude. Full detailed is available in the appendix.
\end{proof}

Proposition \ref{lemma:beta_quantile_bound} and \ref{lemma:beta_quantile_bound_lower} provides lower and upper bound for the Beta quantiles in term of simple functions similar to the one for Bernstein inequalities. We used it to prove our main result in Theorem \ref{thm:bayes_ucrl}.

\begin{repproposition}{lemma:beta_quantile_bound}[Upper bound on the Beta Quantile]
\input{lemma:beta_quantile_bound}
\end{repproposition}
\begin{proof}[Proof Sketch]
These bounds comes directly from the relation between Beta and Binomial cdfs. We apply Proposition \ref{lemma_binomial_quantile_bounds} which gives a bounds for the quantile $p$ in term of $p(1-p)$. We then applies again Proposition \ref{lemma_binomial_quantile_bounds} to bound $p(1-p)$ in term of $\frac{x}{n} \paren*{1-\frac{x}{n}}$. Full proof is available in the appendix.
\end{proof}

\begin{repproposition}{lemma:beta_quantile_bound_lower}[Lower bound on the Beta Quantile]
\input{lemma:beta_quantile_bound_lower}
\end{repproposition}
\begin{proof}[Proof Sketch]
The proof comes almost exclusively by performing the same steps as in the proof of Proposition \ref{lemma:beta_quantile_bound}.
\end{proof}

\section{Experimental Analysis} \label{sec:experiments}
We empirically evaluate the performance of \UCRLJEFFREY{} in comparison with that of \UCRLV{} \citep{full_ucrlv_paper}, \KLUCRL{} \citep{filippi2010optimism} and \UCRL{} \citep{jaksch2010near}. We also compared against \TSDE{} \citep{ouyang2017learning} which is a variant of posterior sampling for reinforcement learning suited for infinite horizon problems. We used the environments \emph{Bandits}, \emph{Riverswim}, \emph{\GAMEOFSKILLEASY{}}, \emph{\GAMEOFSKILLHARD{}} as described in \citet{full_ucrlv_paper}.
We also eliminate unintentional bias and variance in the exact way described in \cite{full_ucrlv_paper}.
Figure~\ref{fig:regrets} illustrates the evolution of the average regret along with confidence region (standard deviation). Figure~\ref{fig:regrets} is a log-log plot where the ticks represent the actual values.
\paragraph{Experimental Setup.} The confidence hyper-parameter $\delta$ of \UCRLV{}, \KLUCRL{}, and \UCRL{} is set to $\ConfidenceDelta{}$. 
\TSDE{} is initialized with independent $\BetaDis(\frac{1}{2}, \frac{1}{2})$ priors for each reward $r(s,a)$ and a Dirichlet 
prior with parameters $(\alpha_1, \ldots \alpha_S)$ for the transition functions $p(.|s,a)$, where $\alpha_i = \frac{1}{S}$. We plot the average regret of each algorithm over $T = 2^{24}$ \rounds{} computed using \TRIALS{} independent trials. 

\paragraph{Implementation Notes on \UCRLJEFFREY{}}
We note here that the quantiles to any subset of next states can be computed efficiently with of a complexity linear in $SA$ and not the naive exponential complexity. This is because 
The posterior to any subset of next states only depend on the sum of the rewards of its constituent.

\begin{figure*}[t!]
	\centering
	\begin{subfigure}{0.45\textwidth}
		\includegraphics[width=\textwidth]{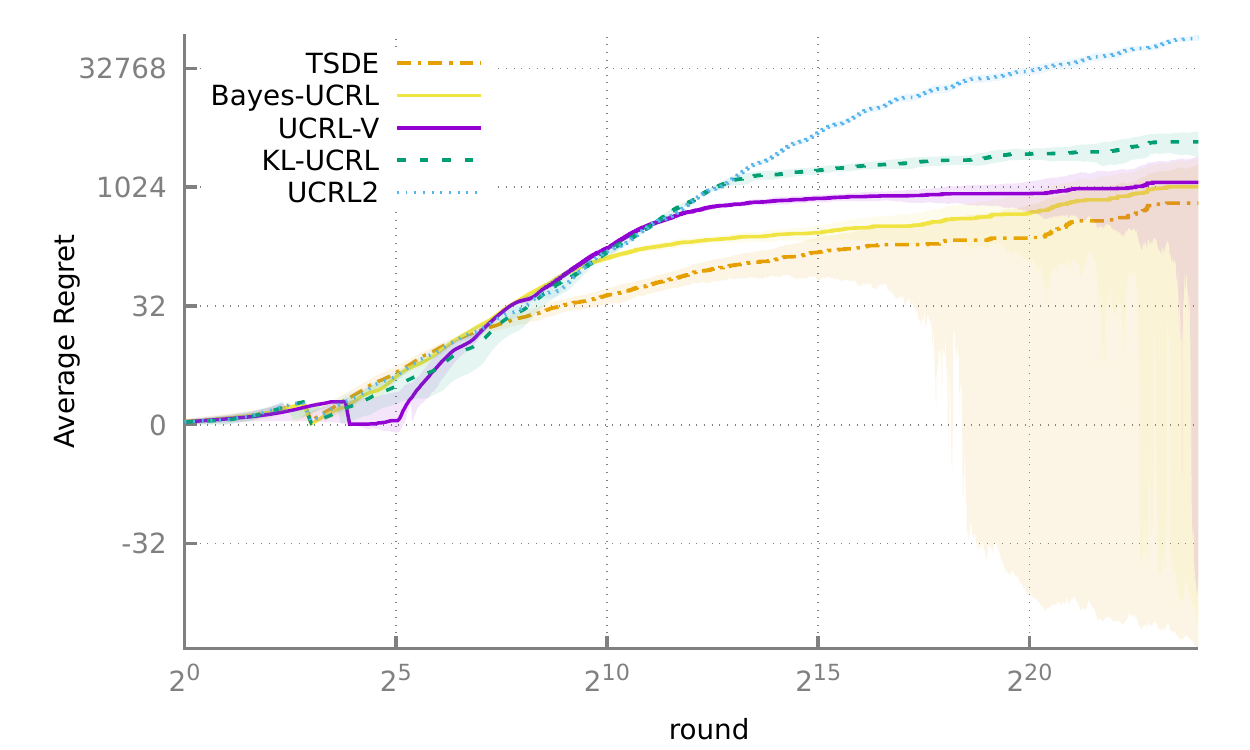}
		\caption{RiverSwim}
		\label{fig:RiverSwim}
	\end{subfigure}%
	\begin{subfigure}{0.45\textwidth}
		\includegraphics[width=\textwidth]{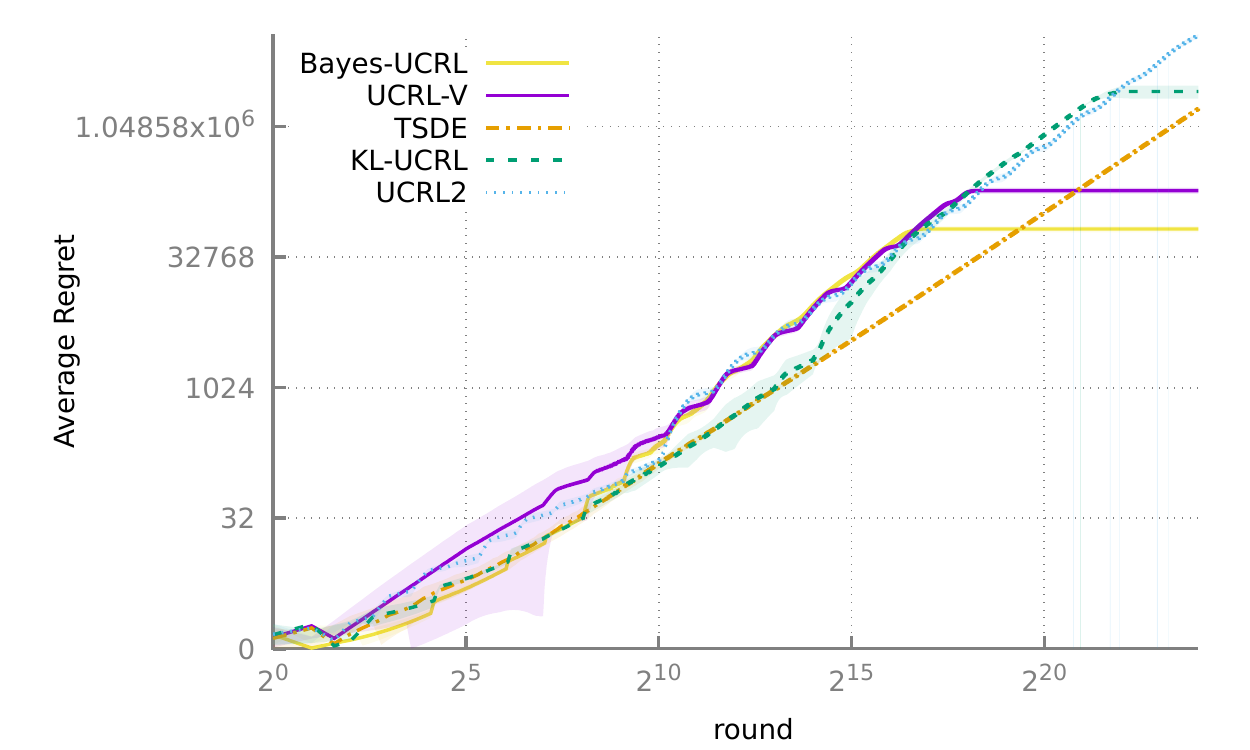}
		\caption{GameOfSkill-v1}
		\label{fig:GameOfSkill-t}
	\end{subfigure}\\
	\begin{subfigure}{0.45\textwidth}
		\includegraphics[width=\textwidth]{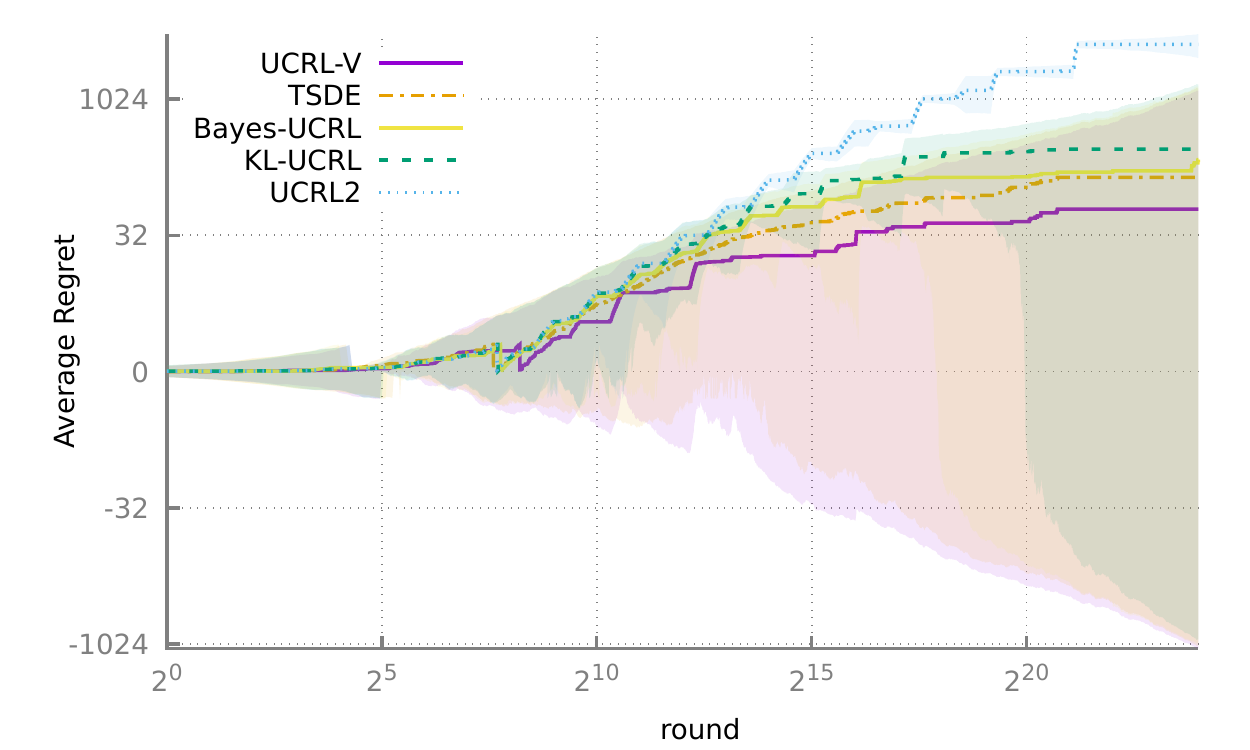}
		\caption{Bandits}
		\label{fig:bandit}
	\end{subfigure}%
	\begin{subfigure}{0.45\textwidth}
		\includegraphics[width=\textwidth]{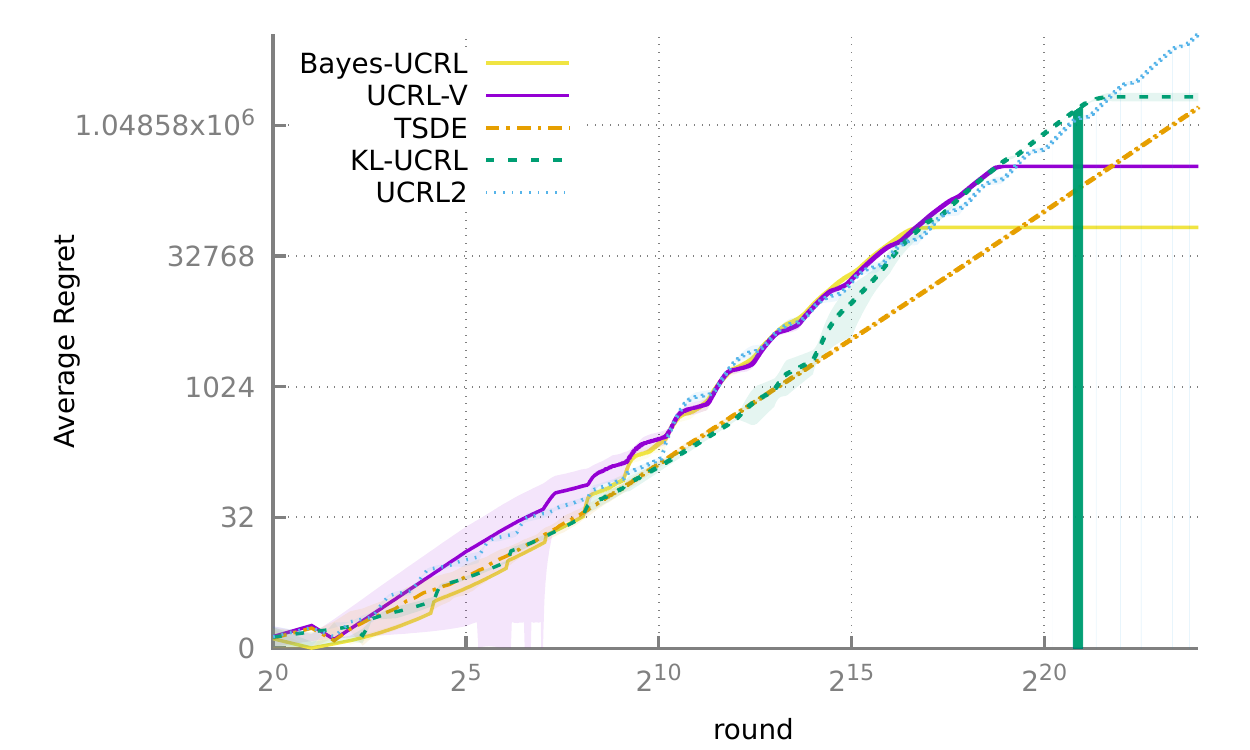}
		\caption{GameOfSkill-v2}
		\label{fig:GameOfSkill-f}
	\end{subfigure}
	\caption{Time evolution of average regret for \UCRLJEFFREY{}, \UCRLV{}, \TSDE{}, \KLUCRL{}, and \UCRL{}.} 
	\label{fig:regrets}
\end{figure*}

\paragraph{Results and Discussion.}

We can see that \UCRLJEFFREY{} outperforms \UCRLBERNSTEIN{} over all environments except in the Bandits one. This is in line with the theoretical regret whereby we can see that using the Bernstein bound is a factor times worse than the Bayesian quantile. Note that this is not an artifact of the proof. Indeed, pure optimism can be seen as using the \emph{proof} inside the algorithm whereas the Bayesian version provides a general algorithm that has to be proven separately. Consequently, the actual performance of the Bayesian algorithm can often be much better than the bounds provided.

\section{Conclusion}
\label{sec:conclusion}
In conclusion, using Bayesian quantiles lead to an algorithm with strong performance while enjoying the best of both frequentist and Bayesian view. It also provides a conceptually simple and very general algorithm for different scenarios. Although we were only able to prove its performance for bounded rewards in $[0,1]$ and a specific prior, we believe it should be possible to provide proof for other rewards distribution and prior such as Gaussian. As future work, it would be interesting to explore how one can re-use the idea of \UCRLJEFFREY{} for non-tabular settings such as with linear function approximation or deep learning.

\bibliography{rlfastposteriorsampling}
\bibliographystyle{icml2019}

\appendix
\section{Proofs}

\subsection{Proof of Theorem \ref{thm:bayes_ucrl}}
Our proof is a direct application of the generic proof provided in Section B.2 of \citet{full_ucrlv_paper}. To use that generic proof we need to show 
that with high probability the true rewards/transitions of any state-action is contained in the lower and upper interval of the Bayesian Posterior. This is a direct consequence of Lemma \ref{lemma:coverage_beta} and the fact that our posterior matches the Beta Distribution used in Lemma \ref{lemma:coverage_beta}. 

Furthermore, we need to provide lower and upper bounds for the maximum deviation of the Bayesian posteriors from their empirical values. This comes directly from using  Proposition \ref{lemma:beta_quantile_bound} and Proposition \ref{lemma:beta_quantile_bound_lower}, and bounding $\Phi^{-1}$ using  equation (15) in \citet{chiani2003new}.

\begin{lemma}[Coverage probability of Beta Quantile for Bernoulli random variable]
	\label{lemma:coverage_beta}
  Let $\bm{X}_1, \ldots \bm{X}_n$ be $n$ independent Bernoulli random variable with common parameter $\mu$ such that $0 < \mu < 1$ and $n \geq 1$. Let $\bm{X} = \sum_{i=1}^{n} \bm{X}_i$ denote the corresponding Binomial random variable.
  Let $U_{1-\delta}(\bm{X})$ the (random) $1-\delta$th quantile of the distribution $\BetaDis(\bm{X} + 1, n-\bm{X})$ and $U_{\delta}(\bm{X})$ the  $\delta$th quantile of the distribution $\BetaDis(\bm{X}, n-\bm{X}+1)$.
  If $0 < \bm{X} < n$, we have:
  \[
    \Prob\left[U_{\delta}(\bm{X}) \leq \mu \leq U_{1-\delta}(\bm{X}) | \mu\right] \geq 1-2\delta.
  \]
\end{lemma}

\begin{proof}
Since each $\bm{X_i}$ is a Bernoulli random variable with parameter $\mu$, then $\bm{X} = \sum_{i=1}^{n} \bm{X}_i$ is a Binomial random variable with parameter $(n, \mu)$.
According to \citet{thulin2014cost} equation (4)
 the quantile of the Beta distribution used in this lemma corresponds exactly to the upper one sided Clopper–Pearson interval (for Binomial distribution) whose coverage probability is at least $1-\delta$ by construction \citep{thulin2014cost}. The same argument holds for the lower one sided Clopper–Pearson interval. Combining them concludes the proof.
\end{proof}

\begin{proposition}[Lower and Upper bound on the Binomial Quantile]
  \label{lemma_binomial_quantile_bounds}
\input{proposition:binom_bounds}
\end{proposition}

\begin{proof}
Using basic computation, we can verify that the bounds hold trivially for $p=0$, for $p=1$ and $n=0$. Furthermore, it is known that any median $m$ of the binomial satisfies $\floor{np} \leq m \leq \ceil{np}$ \citep{kaas1980mean}. So, our bounds also holds for $1-\delta = 0.5$. 
As a result, we can focus the proof on the case where $0 < p < 1$, $n > 0$ and $0.5 \leq 1-\delta < 1$.

From equation (1) in \cite{zubkov2013complete} we have:
\begin{align}
&\Phi\left(\sign(\frac{k}{n}-p)\sqrt{2n\kldiv{\frac{k}{n}}{p}}\right)\notag\\
& \leq \Prob\{\bm{X}_{n,p} \leq k\}\label{eq:binom_cdf_bound}\\
& \leq \Phi\left(\sign(\frac{k+1}{n}-p)\sqrt{2n\kldiv{\frac{k+1}{n}}{p}}\right)\notag
\end{align}
for $0 \leq k <  n$. Let's also observe that when $k = n$, the lower bound in \eqref{eq:binom_cdf_bound} trivially holds since 

\[\Prob\{\bm{X}_{n,p} \leq k\} = 1 \geq \Phi\left(\sign(\frac{k}{n}-p)\sqrt{2n\kldiv{\frac{k}{n}}{p}}\right).\]

\paragraph{Proof of the upper bound}
Our upper bound provides a correction to the Theorem 5 in \citet{short2013improved}.

Consider any $k$ ($0 \leq k \leq n$) such that:
\begin{align}
\Phi\left(\sign(\frac{k}{n}-p)\sqrt{2n\kldiv{\frac{k}{n}}{p}}\right) \geq 1-\delta. \label{eq:upper_binom_upper_goal}
\end{align}
Combining \eqref{eq:upper_binom_upper_goal} with the left side of \eqref{eq:binom_cdf_bound} we have that $ \Prob\{\bm{X}_{n,p} \leq k\} \geq 1-\delta$ and as a result:
\begin{align}
\quant(\Binomial(n,p),1-\delta) &= \inf\{x:\Prob\{\bm{X}_{n,p} \leq x\} \geq 1-\delta\} \leq k
\end{align}

So we just need to find a value $k$ satisfying \eqref{eq:upper_binom_upper_goal}. Remarking that $\Phi^{-1}$ is the CDF of the normal distribution (since it is the inverse of the normal quantile) we can conclude that $\Phi^{-1}$ is continuous and increasing.  Applying $\Phi^{-1}$ to \eqref{eq:upper_binom_upper_goal}, we have:
\begin{align}
\sign(\frac{k}{n}-p)\sqrt{2n\kldiv{\frac{k}{n}}{p}} \geq \Phi^{-1}(1-\delta). \label{eq:goal_inverse_applied}
\end{align}

\subparagraph{The sign of $\frac{k}{n}-p$:}
Assume that $\quant(\Binomial(n,p),1-\delta) \leq \floor{np}$. In that case, we can see that our upper bound trivially holds since $C_u(x,y) \geq \sqrt{\frac{(1-2x)^2y^4}{36}} + \frac{(1-2x)y^2}{6} \geq \abs*{\frac{(1-2x)y^2}{6}} + \frac{(1-2x)y^2}{6} \geq 0$. Then we can focus on the case where $\quant(\Binomial(n,p),1-\delta) > \floor{np}$. Since the binomial distribution is discrete with domain the set of integers, $\quant(\Binomial(n,p),1-\delta) > \floor{np}$ implies that $\quant(\Binomial(n,p),1-\delta) \geq \floor{np} + 1$.
As a result we have $k \geq \quant(\Binomial(n,p),1-\delta) \geq \floor{np} + 1 > np$ and $\sign(\frac{k}{n}-p) = 1$.

Let $x$ a number such that $\frac{k}{n} = p + x$. Using this in \eqref{eq:goal_inverse_applied}, we thus need to find an $x$ such that:

\[ \kldiv{p+x}{p} \geq \frac{\Phi^{-1}(1-\delta)^2}{2n}   \]

Consider a function $g$ such that $\kldiv{p+x}{p} \geq g(x)$. If we find an $x$ such that $g(x) \geq \frac{\Phi^{-1}(1-\delta)^2}{2n}$, then it would mean that $\kldiv{p+x}{p} \geq \frac{\Phi^{-1}(1-\delta)^2}{2n}$. We will pick $g$ to be the lower bound on $\kldiv{p+x}{p}$ in Theorem \ref{theo:kl_bounds}.
Now let's observe that since  $\sign(\frac{k}{n}-p) = 1$, it means that $x \geq 0$. Also $q-x = 1-p-x = 1 -\frac{k}{n} \geq 0$ so that $x \leq q$.

So the condition of  Theorem \ref{theo:kl_bounds} are satisfied and our goal becomes finding an $x \geq 0$ such that:

\[ \frac{x^2}{2(pq+x(q-p)/3)}  \geq \frac{\Phi^{-1}(1-\delta)^2}{2n}   \]
Solving for this inequality leads to the upper bound part of the Theorem.

\paragraph{Proof of the Lower bound}
If $\quant(\Binomial(n,p),1-\delta) = n$, it is easy to verify that our lower bound trivially holds. So we can focus on the case where $\quant(\Binomial(n,p),1-\delta) < n$.

Consider any $k$ ($0 \leq k < n$) such that:
\begin{align}
\Phi\left(\sign(\frac{k+1}{n}-p)\sqrt{2n\kldiv{\frac{k+1}{n}}{p}}\right) \leq 1-\delta \label{eq:lower_binom_upper_goal}
\end{align}

Combining \eqref{eq:lower_binom_upper_goal} with the right side of \eqref{eq:binom_cdf_bound} we have that $ \Prob\{\bm{X}_{n,p} \leq k\} \leq 1-\delta$ and since the CDF of a Binomial is an increasing function, we have:
\begin{align}
\quant(\Binomial(n,p),1-\delta) &\geq k \label{eq:binom_lower}
\end{align}

\subparagraph{The sign of $\frac{k+1}{n}-p$:}
Let's note that the quantile function of the binomial distribution is increasing (since it is the inverse of the cdf and the cdf is increasing). So, we have: $\quant(\Binomial(n,p),1-\delta) \geq \quant(\Binomial(n,p),\frac{1}{2})$.
As a result, there exists a number $k$ satisfying both \eqref{eq:binom_lower} and:

\[\quant(\Binomial(n,p),\frac{1}{2}) \leq k.\]

We will try to find this number. Let's observe that $\quant(\Binomial(n,p),\frac{1}{2})$ is the (smallest) median of the binomial distribution and thus we have:
$\quant(\Binomial(n,p),\frac{1}{2} \geq \floor{np}$ \citep{kaas1980mean}.

So,
\begin{align}
k &\geq \quant(\Binomial(n,p),1-\delta)\\
&\geq \quant(\Binomial(n,p),\frac{1}{2})\\
&\geq \floor{np}\label{eq:binom_median_bound}
\end{align}

As a result, we have $k+1 \geq \floor{np} + 1 > np$ and $\sign(\frac{k+1}{n}-p) = 1$.


Then our objective is to find a $k\geq \floor{np}$ satisfying \eqref{eq:lower_binom_upper_goal}. Let $x$ a number such that $\frac{k+1}{n} = p + x$.

Applying the inverse $\Phi^{-1}$ to \eqref{eq:lower_binom_upper_goal} and replacing $\frac{k+1}{n}$ by $x$, our objective becomes finding an $x \geq 0$ such that:
\[ \kldiv{p+x}{p} \leq \frac{\Phi^{-1}(1-\delta)^2}{2n}   \]

Our objective is equivalent to finding an $x$ such that $g(x) \leq \frac{\Phi^{-1}(1-\delta)^2}{2n}$ for a function $g$ such that $D(p+x,p) \leq g(x)$.

We can easily verify that $x\geq 0$ and $x \leq q$ ($q-x = 1-p-x = 1 - \frac{k+1}{1} \geq 0$). And as a result, we pick $g$ as the first upper bound on $\kldiv{p+x}{p}$ in Theorem \ref{theo:kl_bounds}.

Our objective is thus to find $x$ ($ 0 \leq x \leq 1-p$) such that:
\[ \frac{x^2}{2(pq-xp/2)}  \leq \frac{\Phi^{-1}(1-\delta)^2}{2n}   \]

Solving for this equation and picking a value for $x$
such that $0 \leq x \leq 1-p$, $k \geq \floor{np}$ leads to the first lower bound part of the Theorem.

\end{proof}

\begin{fact}[See \cite{Kaufmann12onbayesian}]
\label{fact_beta_binomial}
Let $\bm{Y}_{a,b} \sim \BetaDis(a,b)$ where $a$ and $b$ some integers such that $a > 0, b>0$ a random variable from the Beta distribution. Then, for any $p \in [0,1]$:
\begin{align}
\Prob(\bm{Y}_{a,b} \leq p) &= \Prob(\bm{X}_{a+b-1, 1-p} \leq b-1)\label{fact_beta_binomial_upper}\\
\Prob(\bm{Y}_{a,b} \geq p) &= \Prob(\bm{X}_{a+b-1, p} \leq a-1)\label{fact_beta_binomial_lower}
\end{align}

where $\bm{X}_{n,x}$ is used to denote a random variable distributed according to the binomial distribution of parameters $(n,x)$ ( $\bm{X}_{n,x} \sim \Binomial(n,x)$).
\end{fact}

\begin{proposition}[Upper bound on the Beta Quantile]
\label{lemma:beta_quantile_bound}
\input{lemma:beta_quantile_bound}
\end{proposition}
 
\begin{proof}
For simplicity, in this proof we used $p = Q(\BetaDis(x+1, n-x), 1-\delta)$ and $y = \Phi^{-1}(1-\delta)$.
Using Equation \eqref{fact_beta_binomial_upper}, we have:

$\Prob[\bm{Y}_{x+1, n-x} \leq p] = \Prob[\bm{X}_{n,1-p} \leq n-x-1]$. Since the CDF of the beta distribution is continuous, we know that $\Prob[\bm{Y}_{x+1, n-x} \leq p] = 1-\delta$.
So we have $\Prob[\bm{X}_{n,1-p} \leq n-x-1] = 1-\delta$

Using the upper bound for Binomial quantile in Lemma \ref{lemma_binomial_quantile_bounds}, we have:

$n-x-1 \leq n(1-p) + C_u(1-p,\Phi^{-1}(1-\delta)) + 1$ where $C_u$ is the function defined in \eqref{eq:binom_quant_fun}.

This leads to:
\begin{align}
\label{eq:upper_bound_p}
p &\leq \frac{x}{n} + \frac{C_u(1-p,y)+2}{n}= \frac{x}{n} + \sqrt{\frac{p(1-p)y^2}{n} + \frac{(2p-1)^2y^4}{36n^2} } +\frac{(2p-1)y^2}{6n} + \frac{2}{n}
\end{align}

We would like to find an upper bound for $p(1-p)$ in \eqref{eq:upper_bound_p} that depends on $\frac{x}{n}(1-\frac{x}{n})$.

Using Equation \eqref{fact_beta_binomial_upper} with the lower bound for binomial quantile in Lemma \ref{lemma_binomial_quantile_bounds}, we have

\begin{align}
\label{lower_bound_p}
p &\geq \frac{x}{n} + \frac{\max\left\{0,\min\left\{np-1, C_l(1-p,\Phi^{-1}(1-\delta))\right\}\right\}}{n} \geq \frac{x}{n}
\end{align}

Multiplying  equations \eqref{lower_bound_p} and \eqref{eq:upper_bound_p} together (both are all positive)
leads to:

\begin{align}
p(1-p) &\leq \paren*{1-\frac{x}{n}}\paren*{\frac{x}{n} + \sqrt{\frac{p(1-p)y^2}{n} + \frac{(2p-1)^2y^4}{36n^2} } +\frac{(2p-1)y^2}{6n} + \frac{2}{n}}
\end{align}
Using the fact that $\sqrt{a+b} \leq \sqrt{a} + \sqrt{b}$, $2p-1 \leq 1$ and using $1-\frac{x}{n} \leq 1$ for the terms not involving $\frac{x}{n}$, we have:

\begin{align}
p(1-p) &\leq \paren*{\frac{x}{n}}\paren*{1-\frac{x}{n}} + \sqrt{\frac{p(1-p)y^2}{n}} + \frac{1}{3} \frac{y^2 + 6}{n}\label{eq:p_1_p_poly}
\end{align}

Letting $z= \sqrt{p(1-p)}$ in \eqref{eq:p_1_p_poly} leads to an inequality involving a polynomial of degree $2$ in $z$. Solving for this inequality and then using $\sqrt{a+b} \leq \sqrt{a} + \sqrt{b}$:

\begin{align}
\sqrt{p(1-p)} &\leq \sqrt{ \paren*{\frac{x}{n}}\paren*{1-\frac{x}{n}}} + \frac{1}{\sqrt{n}} \paren*{\sqrt{\frac{7y^2+24}{12}} + \sqrt{\frac{y^2}{4}}}\label{lower_bound_quantile_variance}
\end{align}

Replacing \eqref{lower_bound_quantile_variance} into \eqref{eq:upper_bound_p} and using the fact that $2p-1 \leq 1$, we have the desired upper bound of the lemma
\end{proof}

\begin{proposition}[Lower bound on the Beta Quantile]
\label{lemma:beta_quantile_bound_lower}
\input{lemma:beta_quantile_bound_lower}
\end{proposition}

\begin{proof}
Let's denote $p = Q(\BetaDis(x, n-x+1), \delta)$. Using \eqref{fact_beta_binomial_lower}, we have that: $\Prob(\bm{Y}_{x,n-x+1} \geq p) = \Prob(\bm{X}_{n, p} \leq x-1)$.

Since the Beta distribution is continuous and also have a continuous cdf, then there exists a unique $p$ such that:

$\Prob(\bm{Y}_{x,n-x+1} \geq p) = 1 - \Prob(\bm{Y}_{x,n-x+1} \leq p) = 1 -\delta$.

As a result, we have $\Prob(\bm{X}_{n, p} \leq x-1) = 1-\delta$.

Using the upper and lower bound for Binomial quantile in Lemma \ref{lemma_binomial_quantile_bounds}, we have respectively:

\begin{align}
\label{eq:upper_bound_p_lower}
p &\geq \frac{x}{n} - \frac{C_u(p,y)+2}{n}= \frac{x}{n} - \sqrt{\frac{p(1-p)y^2}{n} + \frac{(1-2p)^2y^4}{36n^2} } -\frac{(1-2p)y^2}{6n} - \frac{2}{n}
\end{align}

\begin{align}
\label{lower_bound_p_lower}
p &\leq \frac{x}{n} - \frac{ C_l(p,\Phi^{-1}(1-\delta))}{n} \leq \frac{x}{n}
\end{align}

We would like to find a lower bound for $p(1-p)$ in \eqref{eq:upper_bound_p_lower} that depends on $\frac{x}{n}(1-\frac{x}{n})$.

\eqref{lower_bound_p_lower} implies that:
\begin{align}
1-p \geq 1- \frac{x}{n} \label{eq:lower_1_p}
\end{align}

Note that we can multiply \eqref{eq:lower_1_p} by \eqref{eq:upper_bound_p_lower} to get a lower bound for $p(1-p)$ even if the left hand side of \eqref{eq:upper_bound_p_lower} is negative since both $p(1-p)$ and $1-\frac{x}{n}$ are always positive.

 After this multiplication, we follow the exact same steps as in the equivalent part of the proof for Lemma \ref{lemma:beta_quantile_bound}. We can do that since even if we are looking for a lower bound, all the term previously upper bounded in Lemma \ref{lemma:beta_quantile_bound} are multiplied by $-$.
 
 This completes the proof for this lemma.
\end{proof}

\subsection{Useful Results}

\begin{proposition}[Lower and Upper Bound on Bernoulli KL-Divergence]
\label{theo:kl_bounds}
\input{proposition:kl_bounds} 
\end{proposition}
\begin{proof}
The proof of the lower bound already appear in~\citet{janson2016large} (after equation (2.1)).

First, let's observe that:

\[\kldiv{p+x}{p} = p(1 + \frac{x}{p}) \ln(1 + \frac{x}{p}) + h(x)\]

with \[h(x)=
\begin{cases}
q(1-\frac{x}{q}) \ln(1-\frac{x}{q}) & \text{ if }x < q\\
0 & \text{ if } x = q\\
\end{cases}\]

Note that this is a valid definition for the KL-divergence since for any $q \in ]0,1[$, \[\lim_{x \to q^-} (1-\frac{x}{q}) \ln(1-\frac{x}{q}) = 0\]

Let $g(x)$ a parametric function defined by:

\[g(x) = p(1 + \frac{x}{p}) \ln(1 + \frac{x}{p}) + h(x) - \frac{x^2}{2(pq+x(q-p-a)/b)} \]

Where $a$ and $b$ are constants (independent of $x$ but possibly depending on $p$) such that $pq+x(q-p-a)/b > 0$ for all $q \in ]0,1[, x \in [0,q]$.

We can immediately see that $g$ is continuous and differentiable in its domain $[0, q]$ since it is the sum of continuous and differentiable functions.
For any $x \in [0,q[$, the derivative $g'(x)$ of $g(x)$ is:

\[g'(x) = \ln(1+ \frac{x}{p}) -\ln(1-\frac{x}{q}) - \frac{4pqx + 2x^2(q-p-a)/b}{\paren*{2(pq+x(q-p-a)/b)}^2} \]

And $g'(0) = 0$.

We can see that $g'$ is a continuous and differentiable in $[0, q[$.

The second derivative for any $x \in [0,q[$ is 

\begin{align}
g''(x) &= \frac{1}{x+p} + \frac{1}{x+q} - \frac{p^2q^2}{(pq+x(q-p-a)/b)^3}\\
&=\frac{\frac{x^3(q-p-a)^3}{b^3} + \frac{3pqx^2(q-p-a)^2}{b^2} +p^2q^2x^2 - p^2q^2xa + (\frac{3}{b}-1)p^2q^2x(q-p-a) }{(x+p) (x+q) (pq+x(q-p-a)/b)^3}\label{eq:second_derivative}
\end{align}

\paragraph{Proof of the Lower bound}
Let's set $a = 0$ and $b = 3$.
In that case we have for any $x \in [0,q[$:

\begin{align}
g''(x) &= \frac{x^3(q-p)^3/27 + pqx^2(q-p)^2/3 +p^2q^2x^2}{(x+p) (x+q) (pq+x(q-p)/b)^3}\\
&\geq \frac{-px^3(q-p)^2/27 + pqx^2(q-p)^2/3 +p^2q^2x^2}{(x+p) (x+q) (pq+x(q-p)/b)^3}\\
&\geq \frac{-pqx^2(q-p)^2/27 + pqx^2(q-p)^2/3 +p^2q^2x^2}{(x+p) (x+q) (pq+x(q-p)/b)^3}\geq 0
\end{align}

Furthermore, elementary calculations leads to $g(0) = g'(0) = 0$.

Since $g'$ is continuous in $[0,q[$ and $g''$ is positive in $[0,q[$, we can conclude that $g'$ is increasing in $[0,q[$. Since $g'(0) = 0$, it means $g'(x) \geq 0$ for any $ x \in [0,q[$. Since $g$ is continuous, $g'(x) \geq 0$ for any $ x \in [0,q[$ means that $g$ is increasing in $[0,q[$. Using the fact that $g(0) = 0$ we have that $g(x) \geq 0$ for any $ x \in [0,q[$. We will now show that $g(x)$ is non-negative at $q$ too. Using the fact that $g$ is continuous at $q$, we have that $\lim_{x \to q^-} g(x) = g(q)$. Also $\lim_{x \to q^-} g(x) \geq 0$ since $g(x) \geq 0$ for any $x \in [0,q[$. And as a result, $g(q) = \lim_{x \to q^-} g(x) \geq 0$
which concludes the proof of the lower bound.

\paragraph{Proof of the first upper bound}
Let $a = q$ and $b = 2$. We want to analyze the sign of the resulting $g''$ over its domain $[0,q[$. For that observe that the denominator of $g''$ is always strictly positive. This means that the sign of $g''$ is the same as the sign of its numerator. Let's denote $g''_0$ the numerator. We have:

\begin{align}
g''_0(x) &= \frac{x^3(-p)^3}{8} + \frac{3qx^2p^3}{4} +p^2q^2x^2 - p^2q^3x - \frac{p^3q^2x}{2}\\
&= x \cdot \paren*{\frac{x^2(-p)^3}{8} + \frac{3qxp^3}{4} +p^2q^2x - p^2q^3 - \frac{p^3q^2}{2}}
\end{align}

Let's denote $f(x) = \frac{x^2(-p)^3}{8} + \frac{3qxp^3}{4} +p^2q^2x - p^2q^3 - \frac{p^3q^2}{2}$. 

We will use Sturm's theorem (Theorem \ref{theo:sturm}) to find the number of roots of $f$ in $]0,q]$.
The Sturm sequence of $f$ is $\set{f_0, f_1, f_2}$ with:
\begin{align}
f_0(x) &= f(x)\\
f_1(x) &= \frac{x(-p)^3}{4} + \frac{3qp^3}{4} +p^2q^2\\
f_2(x) &= \frac{p(-5p^4 + 26p^3 - 53p^2 + 48p - 16)}{8}
\end{align}

We have:
\begin{equation*}
f_0(0) = -p^2q^3-\frac{p^3q^2}{2} < 0
\end{equation*}
\begin{equation*}
f_1(0) = \frac{3qp^3}{4} +p^2q^2 > 0
\end{equation*}

\begin{equation*}
f_0(q) = \frac{p^3q^2}{8} > 0
\end{equation*}
\begin{equation*}
f_1(q) = \frac{qp^3+2p^2q^2}{2} > 0
\end{equation*}

And we have $f_2(q) = f_2(0)$

The number of sign alternations in $\set{f_0(0), f_1(0), f_2(0)}$ is: $1+\1_{f_2(0) < 0}$ where $\1_{f_2(0) < 0} = 1$ if $f_2(0) < 0$ and $\1_{f_2(0) < 0} = 0$ otherwise. The number of sign alternations in $\set{f_0(q), f_1(q), f_2(q)}$ is:$\1_{f_2(0) < 0}$. Observing that neither $0$, nor $q$ are roots of $f$, we can conclude by the Sturm's theorem (Theorem \ref{theo:sturm}) that the number of roots of $f$ in  $[0, q]$ is exactly 1.

Since $f$ is a polynomial it means that the sign of $f$ changes at most once in the interval $[0,q]$. Let's $\alpha$ ($0 < \alpha < q$) the unique root of $f$ in $[0,q]$. Then (ignoring zero-values) the function $f$ have the same sign for all values in $[0,\alpha]$ and $f$ have the same sign for all values in $[\alpha, q]$.

Observing that $f(0) < 0$, it means that $f(x) \leq 0$ for any $x \in [0,\alpha]$. Observing that 
$f(q) > 0$, it means that $f(x) \geq 0$ for any $x \in [\alpha, q]$. Since the second derivative $g''$
is a multiple of a non-negative terms by $f$; it means that $g''(x) \leq 0$  for any $x \in [0,\alpha]$ and $g''(x) \geq 0$ for any $x \in [\alpha, q[$.

We will now derive the sign of $g'$ and $g$ over their domain. Since $g'(0) = 0$, $g''(x) \leq 0$ in $x \in [0, \alpha]$ and $g'$ is continuous in $[0, \alpha]$, we can conclude that $g'$ is decreasing in $[0, \alpha]$ and $g'(x) \leq 0$ for any $x \in [0, \alpha]$. A similar argument for $g$ allows us to conclude that $g(x) \leq 0$ for any $x \in [0, \alpha]$.

Since $g''(x) \geq 0$ for any $x \in [\alpha, q[$, it means that $g'$ is increasing in the interval $[\alpha, q[$. Let $\alpha_0$, the lowest value in $[\alpha, q]$ such that $g'(\alpha_0) = 0$. This means that for any $x \in [\alpha, \alpha_0]$ $g'(x) \leq 0$ and for any $x \in [\alpha_0, q]$, $g'(x) \geq 0$. Since $g'$ is non-positive in $[\alpha, \alpha_0]$ and $g$ is continuous, we have that $g$ is decreasing in $[\alpha, \alpha_0]$ so that $g(x) \leq g(\alpha) \leq 0$. 

If $\alpha_0 = q$, our proof is essentially done since it implies $g(x) \leq 0$ for all $x \in [0,q]$.

Assume that $\alpha_0 < q$. We want to identify the sign of $g$ in $[\alpha_0, q]$. In this case, we know that $g'(x) \geq 0$ so that $g$ is increasing in $[\alpha_0, q]$.

Now let's observe that:

\begin{align}
g(q) &= p(1 + \frac{q}{p}) \ln(1 + \frac{q}{p}) + h(q) - \frac{q^2}{pq}\\
&= \ln(1 + \frac{q}{p}) - \frac{q}{p}\\
&\leq \frac{q}{p} - \frac{q}{p} = 0 \label{eq:g_q_zero}
\end{align}

We will now show by contradiction that $g(x) \leq 0$
for all $x \in [\alpha_0, q]$. Assume that there exists a number $c \in [\alpha_0, q]$ for which $g(c) > 0$. Since $g$ is increasing (and continuous) in $[\alpha_0, q]$, it means that $g(q) > 0$ which contradicts \eqref{eq:g_q_zero}. As a result, there is no value $c \in [\alpha_0, q]$ such that $g(x) > 0$. And this concludes the proof for the first upper bound.

\paragraph{Proof of the second upper bound}
The second upper bound comes directly from the first upper bound and the fact that:
\[2(pq-xp/2) = pq + p (q-x) \geq pq\]


\end{proof}

\section{Previously known results}

\subsection{Sturm's Theorem}

\begin{definition}[Sturm sequence of a univariate polynomial]
For any univariate polynomial $f(x)$ of degree $d$ with real coefficients, the sturm sequence for $f$ is a sequence of polynomials $\bm{\bar{f}} = \set{f_0, f_1 \ldots f_d}$ such that:
\begin{align*}
f_0 &= f\\
f_1 &= f'\\
f_{i+1} &= - f_{i-1} \rem f_i \;\forall i \in \set{0,\ldots, d-2}
\end{align*}
where $f_{i-1} \rem f_i$  denotes the remainder of the euclidian division of $f_{i-1}$ by $f_i$
\end{definition}

\begin{definition}[Number of sign alternations in an arbitrary sequence]
Let $\bm{\bar{\alpha}} = \set{\alpha_0, \alpha_1, \ldots \alpha_n}$ be a sequence of real numbers. We say that there is a sign alternation at position $i \in \set{1, \ldots n}$, if there exists some $j \in \set{0, \ldots i-1}$ such that the following two conditions are satisfied
\begin{enumerate}[label=(\roman*)]
\item $\alpha_i \alpha_j < 0$
\item $j = i-1$ or $\alpha_{j+1} = \alpha_{j+2} = \ldots = \alpha_{i-1} = 0$.
\end{enumerate}

The number of sign alternations of $\bm{\bar{\alpha}}$ is the number of positions for which there is a sign alternation.
\end{definition}

\begin{definition}[Multiplicity of a root of a function]
Let $r \geq 0$ a non-negative integer. Let $f^{(0)} = f$ and $f^{(n)}$ the $n$th derivative of a function $f$ differentiable up to $n$ times. A real number $\alpha$ is called a root of multiplicity $r$ for $f$ if:
\[ f^{(0)}(\alpha) = \ldots = f^{(r-1)}(\alpha) = 0; \;\; f^{(r)}(\alpha) \ne 0 \]

We say a real number \emph{is not a multiple root} if its multiplicity is less or equal to $1$ (i.e. it is either a non-root or it is a root of multiplicity 1).
\end{definition}
\begin{theorem}[Sturm's Theorem \citep{yap2000fundamental, sturm_endpoint}]
\label{theo:sturm}
For any non-zero univariate polynomial $f(x)$ with real coefficients and any numbers $a\leq b$; let $N_f]a,b]$ the number of distincts real roots of $f$  in $]a,b]$. If neither $a$ nor $b$ are multiple roots, We have:

\[N_f]a,b] = S_f(a) - S_f(b)\]

Where $S_f(x)$ denotes the number of sign alternatives obtained for the sequences $\set{f_0(x), f_1(x), \ldots f_d(x)}$ where $\set{f_0, f_1, \ldots, f_d}$ is the sturm sequence of the polynomial $f$.

\end{theorem}

\end{document}